\documentclass{article}
\usepackage{ijcai21}

\usepackage{times}
\usepackage{soul}
\usepackage{url}
\usepackage[hidelinks]{hyperref}
\usepackage[utf8]{inputenc}
\usepackage[small]{caption}
\usepackage{graphicx}
\usepackage{amsmath}
\usepackage{amsthm}
\usepackage{booktabs}
\usepackage{algorithm}
\usepackage{comment}
\usepackage{enumitem}
\usepackage{xcolor}
\usepackage{algpseudocode}

\newtheorem{thm}{Theorem}
\newtheorem{cor}{Corollary}

\title{Sequential Online Chore Division for Autonomous Vehicle Convoy Formation}

\author{
Harel Yedidsion$^1$  \and
Shani Alkoby$^2$\and
Peter Stone$^{1,3}$
\affiliations
$^1$The University of Texas at Austin\\
$^2$Ariel University\\
$^3$Sony AI\\
\emails
harel@cs.utexas.edu,
shania@ariel.ac.il,
pstone@cs.utexas.edu
}

\begin{document}

\maketitle

\begin{abstract}
	Chore division is a class of fair division problems in which some
undesirable ``resource" must be shared among a set of participants, with
each participant wanting to get as little as possible. Typically the
set of participants is fixed and known at the outset. This paper
introduces a novel variant, called sequential online chore division
(SOCD), in which participants arrive and depart online, while the chore
is being performed: both the total number of participants and their
arrival/departure times are initially unknown. In SOCD, exactly one
agent must be performing the chore at any give time (e.g. keeping
lookout), and switching the performer incurs a cost. In this paper, we propose and analyze three mechanisms for SOCD: one centralized
mechanism using side payments, and two distributed ones that seek to
balance the participants' loads. Analysis and results are presented in a domain motivated by autonomous vehicle convoy formation, where the chore is leading the convoy so that all followers can enjoy reduced wind resistance.			
\end{abstract}

\section{Introduction}
\label{sec:intro}
	Autonomous vehicles are said to form a convoy when vehicles headed in the same direction follow each other in close proximity. This behavior has been proven to save energy, due to the reduction in aerodynamic drag, and is used by migrating bird flocks and in cyclist pelotons.\footnote{Another commonly used term for convoy formation is \emph{platooning}. We use these two terms interchangeably.}
Autonomous vehicle technology offers a safe and accurate way of following with short inter-vehicle distances, even at high speeds, thanks to Vehicle-to-Vehicle (V2V) communication, which can alert all the followers immediately whenever any slowing is necessary. Empirical evaluations estimate that a follower can save over $10\%$ of its fuel consumption \cite{lammert2018influences}.
However, since the leader sees little or no such gains, choosing the leader of such a convoy raises issues of fairness. Solving these issues is challenging since vehicles can dynamically join and leave the convoy.

This convoy formation problem is representative of an interesting class of previously unexplored fair division problems.
Fair division is concerned with dividing a resource between several players, such that each one receives a fair share. One of the most notable fair division problems is cake cutting. Chore division is the dual problem, in which an undesirable task must be fairly divided among agents.
Motivated by the details of the convoy formation problem, we define a novel and unique variation of chore division called \textbf{S}equential \textbf{O}nline \textbf{C}hore \textbf{D}ivision (SOCD),
where agents arrive online, their number is not known a priori, and only one agent can handle the chore at any given time. 
We investigate how to design SOCD allocation mechanisms that guarantee fairness, and maximize efficiency.

The notion of fairness has various interpretations such as proportionality, envy-freeness, and equability. Guaranteeing fairness in dynamic environments, where either resources or participants arrive online, is not always possible for a single game \cite{walsh2011online}, for any reasonable definition of fairness, while in repeated games, fairness can be guaranteed in expectation.



To the best of our knowledge, none of the fair division literature in either game theory or multiagent systems has considered the SOCD problem as we define it. Furthermore, no previous work has developed a mechanism for profit sharing or load balancing among the vehicles in a convoy. 

This paper's contribution is thus twofold. First, in the area of fair division, it defines the general SOCD model, and second, in the area of convoy formation and platooning it introduces mechanisms that enable spontaneous formation of ad hoc convoys while maintaining fairness and efficiency.

We find that optimal fairness and efficiency can be guaranteed in a centralized setting. In a distributed setting, they can be guaranteed in expectation after participating in multiple games. However, for a single game in a distributed setting, only a relatively weak form of fairness, ex-ante proportionality, can be guaranteed with minimal efficiency loss. 

Following a review of related work in Section \ref{sec:related_work}, a formal definition of the problem is given in Section \ref{sec:model}. Issues of fairness in online and distributed mechanisms are discussed in Section \ref{sec:fairness}. Section \ref{sec:convoy_model} models the convoy formation problem as an SOCD problem. Our proposed solution mechanisms are defined and analyzed in Section \ref{sec:convoy_formation}. Conclusions and future work are provided in Section \ref{sec:conc}.

\section{Related Work}
\label{sec:related_work}

Fair division is a long-standing and still very active field of research spanning multiple disciplines such as economics, sociology, game-theory and mathematics, and having numerous real-world applications \cite{brams1996fair,friedman2003pricing}. These applications consider the allocation of both goods and chores. Compared to goods, the literature on fair allocation of chores
is relatively under-developed 
\cite{aziz2016computational}.

	In dynamic environments, where either agents or goods arrive online, the problem becomes more complex and even the definition of fairness becomes challenging to specify. Several papers define modified notions of fairness in online settings	
\cite{aleksandrov2015online,kash2014no,friedman2015dynamic,benade2018make}. 
Similar to these papers 
, we also define dynamic fairness criteria for SOCD in Section \ref{sec:fairness}. 

In online cake cutting problems, where agents arrive online, and have heterogeneous valuation functions, 
it has been proven that no online cake cutting procedure is either proportional, envy-free, or equitable \cite{walsh2011online}.
We adhere to this paper's call to continue investigating online chore division. We also extend its analysis by providing an impossibility result for both envy-free, and for equitable allocations in the single game distributed SOCD problem.

The SOCD model is relevant to applications such as assigning a guard to keep a lookout at a campground where travelers arrive online, or assigning a goal keeper in a drop-in soccer match. In this paper, we focus on convoy formation due to its social impact. Autonomous vehicle technology such as 
Cooperative Adaptive Cruise Control (CACC) \cite{lammert2018influences}  utilizes a combination of sensory data and V2V communication to enable vehicles to cooperate and follow each other closely, accurately, and safely, by synchronizing braking and accelerating.

Safe grouping of vehicles into convoys offers numerous advantages including: increased energy efficiency, improved road capacity, increased traffic safety, and decreased harmful emissions.  
As a result, major projects are being undertaken around the world in academia, private fleet companies, auto manufacturers, governments, and by individuals, to develop the applicability and regulation of convoys \cite{shladover2007path,robinson2010operating,tsugawa2013overview,de2014network,englund2016grand,peloton}.

While the percent of fuel saving varies with vehicle weight, size, speed, and inter-vehicle distance, one finding remains consistent across all studies; the leader's savings is significantly lower than that of the followers \cite{al2010experimental,lammert2014effect,lammert2018influences}. 

Many of the research projects and experiments in this area are geared towards single fleet convoys, owned and operated by the same organization \cite{bhoopalam2017planning,bergenhem2012overview,kavathekar2011vehicle,janssen2015truck}, and as a consequence do not put an emphasis on developing ways to fairly divide the otherwise unequal savings between the leader and the followers. As opposed to single fleet convoys where participants are not self-interested and are all motivated to maximize the social welfare, in ad-hoc convoys, individual participants are interested in maximizing their own energy savings. Consequently, to enable the spontaneous formation of ad-hoc convoys, it is imperative to design a mechanism that ensures a fair division of both benefits and duties among convoy participants. 




\section{Sequential Online Chore Division}
\label{sec:model}
	In this section we specify the definitions, assumptions, and constraints of the SOCD problem, where a continuous chore must be divided among an a-priori unknown number of agents. The input to the problem is an online stream of agents $A$, one of which must be tasked with performing the chore at any given time. Agent $i$, denoted as $a_i$, arrives at time $t\_arrive_i$ and leaves at time  $t\_leave_i$. During this period of time, [$t\_arrive_i$,$t\_leave_i$], $a_i$ is considered to be \textit{available}, and able to perform the chore. When performing the chore, an agent is considered to be \textit{active}.
Every available agent, except for the active one, gains a positive utility $u_i$ per unit time. The active agent gains nothing.

In this initial treatment of SOCD, we make the following assumptions, each of which may be relaxed in future work:
\begin{itemize}[leftmargin=*]
	\item Agents are homogeneous and have the same utility per unit of time $ \forall i \quad u_i=u $, the same cost per unit of time when active, $ \forall i \quad ca_i=ca $ (in this analysis we set $ca=0$), and the same valuation for being active for any part of the chore, $ \forall i,j,s \quad V_i(s)=V_j(s)$. We provide a short discussion on the challenges associated with considering heterogeneous agents in Appendix A.
	\item Each agent knows its own arrival and departure times, and communicates this information truthfully to the agents that are present when it arrives.
	\item No two agents have the exact same arrival time.
\end{itemize}	
The time frame for an SOCD game, $T$, starts when there is at least one available agent, and ends when there are none left.

Within one availability period agents may have more than one active period. We denote the $m$-th time at which $a_i$ is assigned to become  \textit{active} by $t\_start^i_{m}$ and the corresponding $m$-th time in which it is assigned to stop by $t\_stop^i_{m}$.
$M_i$ is the set of $a_i$'s active periods. 

Agent $a_i$'s assigned share of the task, denoted as $s_i$, is the sum of all the periods that $a_i$ is assigned to be active, $s_i= \sum_{m \in M_i} t\_stop^i_{m}- t\_start^i_{m}$

We define a switch between an active agent $a_i$ and an available agent $a_j$, to happen if $t\_stop^i_{m}=t\_start^j_{n} \quad m \in M_i, n \in M_j$, for any of their stop and start times respectively. 

A switch results in a cost to the system $c$. We intentionally leave this definition as general as possible, since some applications may assign the cost to the outgoing agent while others to the incoming agent. In some applications this cost is fixed while in others if can be a function that depends on the state of the system, as is the case in convoy formation switching cost which is detailed in Section \ref{sec:convoy_model}.

In SOCD 
a \textit{feasible solution} is an online assignment of one agent to be active, out of the available agents in $A$, at any given time $t \in T$. 

Efficiency $E$ is defined as the total utility gained by all participating agents. 

$E= \sum_{a_i \in A} u \cdot[{t\_leave_i-t\_arrive_i} - s_i] - c \cdot nos$

with $nos$ representing the number of switches during $T$.

\begin{center}
	\textbf{Problem Definition for SOCD:}\\
	\textit{Given an online input stream of agents, find a mechanism that produces a feasible solution while maximizing both efficiency and fairness.}
\end{center}
With the assumption of homogeneity, maximizing efficiency is straightforward; simply reduce the number of switches as much as possible. Maximizing fairness on the other hand is more complex, as explained in Section \ref{sec:fairness}. 





Figure \ref{fig:example} illustrates the concepts that define an SOCD problem. It shows the availability periods of three agents as horizontal lines stretching from their arrival times to their respective leaving times. The period of this single game is between $a_1$'s arrival time and $a_3$'s leave time, and is marked with a darker background. Throughout this period, one agent is active, starting with $a_1$, followed by $a_2$, and then $a_3$, as denoted by the dashed red lines. Note that this allocation is a feasible one, but is not necessarily fair. We will discuss what is considered fair in SOCD in Section \ref{sec:fairness}.
\vspace{-10pt}
\begin{figure}[ht]
	\centering
	\includegraphics[scale=0.22]{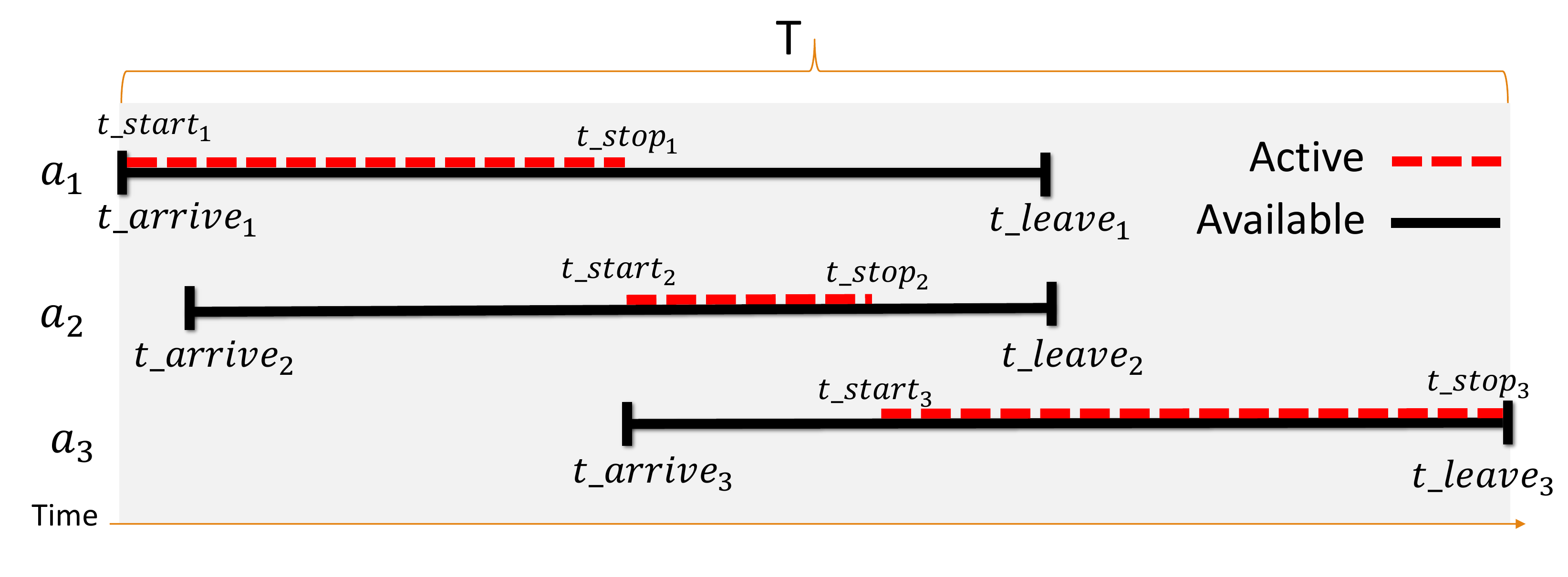}
	\vspace{-15pt}
	\caption{An example of the SOCD model. Three agents' availability periods are displayed on the time axis with agents 1, 2, and 3 acting as the active agent sequentially.}
	\label{fig:example}
\end{figure}
\vspace{-10pt}

\section{Fairness Definitions and Properties}
\label{sec:fairness}

In the SOCD model, when assessing the fairness of a specific agent's allocation, we consider allocations of agents with overlapping availability periods, and distinguish between earlier arrivals and later arrivals, similar to the \textbf{dynamic} definitions of fairness in \cite{kash2014no}.


	For a given agent, $a_i$, we define two notions of proportionality. The first is ex-ante proportionality which takes into account only the agents which are present at $t\_arrive_i$. The second is ex-post proportionality which considers all the agents that were available during $a_i$'s availability period. 


In order to define $a_i$'s proportional share in the dynamic SOCD model, it is necessary to separately analyze every segment of its availability period.
In each segment there is a different subset of available agents from $A$.
We calculate $a_i$'s proportional share for each segment, and finally sum up all of these shares.

	\textbf{Ex-Ante Proportional Share}

We define \emph{EAS}$^i$, as the set of segments, $seg_j^i \in$ \emph{EAS}$^i$ ($j$ is simply an index of the segments) within $a_i$'s availability period, which are known at $t\_arrive_i$. This set does not consider future arrivals. The first segment, $seg_1^i$, starts with $a_i$'s arrival, and ends at the $t\_leave$ of the first agent among the ones that are present at $t\_arrive_i$. Each consecutive segment ends at the departure of another agent, until $t\_leave_i$.  

The ex-ante proportional share for $a_i$ is the sum of known segments' sizes, each divided by the respective number of agents present at that segment, $n\_seg_j^i$, without considering future arrivals. we do not consider future arrivals in the calculation of the ex-ante proportional share since we only have estimates of future arrivals, and thus, in the worst case scenario, the estimate will not reflect the real outcome. In such a case, the sum of the proportional shares will not add up to cover the entire task.
$$ ex\_ante\_prop_i = \sum_{j=1}^{j=|EAS^i|} \frac{ |seg_j^i|}{ n\_seg_j^i} + \frac{c}{u} $$ 

We add one switching cost to the proportional share of every agent since in the worst case, every agent except for the last one would have to switch at least once in order to divide the chore into $n$ parts. The proportional share is expressed in terms of time. In order to keep the unit of measurement consistent, the switching cost, $c$, is divided by the utility per unit time, $u$.

	\textbf{Ex-Post Proportional Share}

The ex-post proportional share of $a_i$ considers the actual segments, including future arrivals, that occurred during $a_i$'s availability period. 
We define \emph{EPS}$^i$ as the set of actual segments within $a_i$'s availability period, $seg_j^i \in$ \emph{EPS}$^i$. The first segment, $seg_1^i$, starts with $a_i$'s arrival, and a new segment starts whenever there is a change in the number of available agents, until $t\_leave_i$.  

$$ ex\_post\_prop_i = \sum_{j=1}^{j=|EPS^i|} \frac{ |seg_j^i|}{ n\_seg_j^i} + \frac{c}{u} $$ 

Note that ex-post proportionality can only be calculated in retrospect, while the ex-ante proportionality can be calculated immediately when the agent arrives. 
If no new agents arrive during $a_i$'s availability period, its ex-ante and ex-post shares are the same.
Also note that the ex-post proportional share is also envy-free and equitable since for each segment all the agents get equal shares.

Figure \ref{fig:prop} provides an example which highlights the proportional shares of agents $a_1$,  $a_2$ and $a_3$. The three horizontal lines represent the availability periods of the agents. 
Agent $a_2$'s ex-ante proportional share is a half of its availability period which is shared with the existing agent $a_1$ and the total rest of its availability period. This share is highlighted in purple on top of $a_2$'s timeline. 
Agent $a_2$'s ex-post proportional share is a half of its availability period which is shared with just $a_1$, a third of its shared availability with both $a_1$ and $a_3$, a half of the time with just $a_3$, and the rest of the time alone. These shares are highlighted by the orange below $a_2$'s timeline.
 
There are no new entrants after $a_3$ and so its ex-ante and ex-post shares are the same.
Note that the switching costs are not depicted in this diagram. However, the cost (in terms of time) of one switch is added to the ex-ante proportional share of each agent.

\begin{figure}[h]
	\centering
	\includegraphics[scale=0.33]{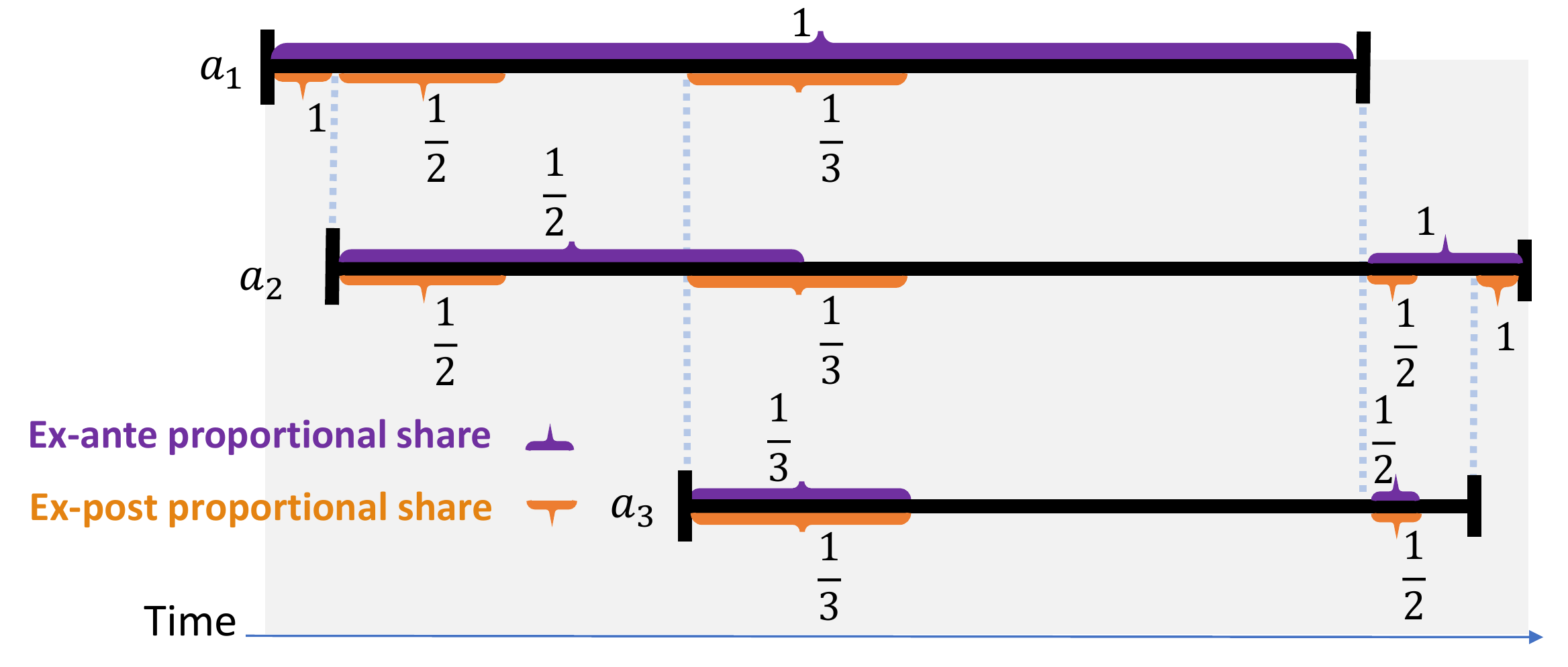}
	\caption{An example of the calculation of the proportional shares.}
	\label{fig:prop}
	\vspace{-5pt}
\end{figure}

Theorem \ref{thm:imposibility} outlines the limitations of guaranteeing equitability (and envy-freeness) in SOCD.


\begin{thm}\label{thm:imposibility}
	In SOCD, no mechanism can guarantee ex-post proportionality for a single game, in a distributed setting.
\end{thm}

\begin{proof}
	Since the chore can only be performed sequentially by one agent at a time, any schedule has to assign one agent, $a_{last}$, to be the last to perform the task, and complete its equitable share $s_{last}$. By performing the last part of the task, $a_{last}$ would satisfy the other agents' demands for equitability, i.e., doing the same as they did, $\forall i,j$ $s_i = s_j$. Specifically having $\forall i$ $s_i = s_{last}$. However, because we are dealing with an online arrival stream of agents, it is possible that a new agent, $a_{new}$ would arrive just when agent $a_{last}$ is about to perform its last portion of the task. The new agent would share that portion with $a_{last}$ and would reduce $a_{last}$'s share. Thus, having  $\forall i$ $s_{last} < s_i$. Note that if $a_{new}$ does not contribute anything then that would create inequitability with regards to $a_{new}$.
	Hence, equitability cannot be guaranteed ex-post (nor envy-freeness). 
\end{proof}

\section{Modeling Convoy Formation as SOCD}
\label{sec:convoy_model}
In this section we frame the convoy formation problem as an instance of the SOCD problem, and outline the application-specific assumptions that are relevant to convoy formation. 

Initially, in convoy formation, the \emph{chore} to be divided is \emph{leading} the convoy and the $active$ agent in the SOCD model is the $leading$ agent in the convoy.\footnote{ Despite  evidence showing that the leading agent might get some reward, we assume that it gains nothing as in the general SOCD model since its gain is negligible compared to that of the followers and thus does not meaningfully affect our analysis.
}


\paragraph{\textbf{Switching and Rotating}}

	Unlike the general definition of switching an active agent in the SOCD model, in the convoy formation setting, we distinguish between two types of switching: rotating and joining/leaving.
A \emph{rotation} is when a leading agent finishes its share and moves to the back of the convoy.
The rotation process requires the leading agent to switch lanes, slow down to let the convoy pass it, and rejoin from the back. 
The time it takes to rotate is proportional to $n_r$ - the number of vehicles in the convoy at the time of rotation. 
During the rotation process, the rotating agent is effectively out of the convoy and does not enjoy fuel savings. Therefore, this form of switching incurs a cost to the rotating agent. 
Other forms of switching, happen when a new agent joins the convoy at the front, or when the leader leaves entirely. We assume that these forms of switching do not incur any cost.
Formally, in convoy formation, the switching cost, $c_{cf}$, incurred by outgoing leading agent $a_i$, when switching with $a_j$, is defined as:
$$
c_{cf}(n_r)  =
\begin{cases}
\text{0} &\quad\text{if $a_j$ joins from the front}\\
\text{0} &\quad\text{if $a_i$ leaves}\\

\text{$c \cdot n_r$} &\quad\text{otherwise} \\ 
\end{cases}
$$

%
%

\paragraph{\textbf{The constant speed assumption}}   

For simplicity, we assume that the convoy is moving at a constant speed and so the time spent in the convoy is proportional to the length of the road traveled. 
This assumption implies that we can refer to $u$ which is the utility per unit time, also as the utility per unit length.
We expect our results to easily generalize as long as speed limits are known for all road segments.
%

\section{Convoy Formation Mechanisms}
\label{sec:convoy_formation}
In this section we outline three possible convoy formation mechanisms, each geared toward a different set of environmental assumptions. The first mechanism is applicable when there is a central payment transfer system. The second mechanism assumes a distributed setting where there are no payment transfer abilities, and guarantees fairness in expectation through repeated games. The third mechanism also assumes a distributed setting, and aims to guarantee fairness for every single game by rotating the leader. 
The last case is the interesting one from a technical point of view.

\subsection{Payment Transfers}

Assuming the existence of a central payment transfer system which can charge or refund the agents, the Payment Transfer mechanism (Mechanism \ref{alg:pt}) assigns only one active agent while the followers transfer a share of their savings to the leader in order to keep fairness. They do so at every time there is a change in the number of agents in the convoy, i.e., at the end of every segment in $EPS^i$. 
We divide the chore $X$ into segments where in each segment the number of convoy members is constant (i.e., a segment is a part of the road between two adjacent arrivals/departures). We denote the number of convoy members in segment $seg$ as $n_{seg}$. For each segment, the amount that a following agent $a_i$ needs to pay the leader of that segment is $p_{seg}^i = \frac{|seg|\cdot u}{n_{seg}}$. 


This mechanism allows agents to join the convoy from the rear, middle, or front, and does not require any rotations to be made at all, yielding an optimal solution in terms of efficiency.
It is also optimal in terms of fairness since the agents equally share the savings for every segment. Each agent pays a cost which is equivalent to the loss of savings it would have absorbed if it had lead for its equitable share, i.e., its ex-post proportional share. Furthermore, the leader gets payments which are equal to the saving it would have enjoyed if it had only led for its equitable share, as proved in Theorem \ref{prop:ProfitSharingOptimality}.
\begin{thm}\label{prop:ProfitSharingOptimality}
	The Payment Transfer mechanism is efficiently optimal, and equitable.
\end{thm}

The proof can be found in Appendix B.

\begin{algorithm}
\small
	\caption{ The Payment Transfer Mechanism }\label{alg:pt}
	\begin{itemize}[leftmargin=5pt]

		\item Agents can join from the rear, middle, or front of the convoy.
		\item A switch happens when the leader leaves or when a new agent joins from the front.
		\item For every segment, each following agent transfers $p_{seg}^i$ to the leader of that segment.
	\end{itemize}
\end{algorithm}

Although we are mainly interested in the distributed setting, and although the design and analysis of this centralized mechanism is fairly straightforward, we include it since it guarantees ex-post proportionality, a result which based on Theorem \ref{thm:imposibility} is unattainable for distributed mechanisms. Distributed mechanisms strive to produce allocations which are as close as possible to the ones produced by the payment transfer mechanism. 

\subsection{Load-Balancing}
The payment transfer mechanism's basic assumption is the availability of some payment system to each one of the convoy's participants. In real-world scenarios however, this is not always the case. 


As a result, in this section we consider a load balancing mechanism aimed to distribute the load equally between all convoy participants. 



\subsubsection{Repeated Game Load Balancing}

In the distributed setting, we start by analyzing a mechanism where agents do not make any costly rotations.

The Repeated Game Load Balancing mechanism that we propose (Mechanism \ref{alg:rglb}) requires that each agent first contributes its share and only then will enjoy the advantages of being a follower. Therefore, any new agent can only join the convoy from the front, and become the leader until someone else joins, or until it leaves.

\begin{algorithm}[h]
\small
	\caption{Repeated Game Load Balancing Mechanism}\label{alg:rglb}
	\begin{itemize}[leftmargin=5pt]

		\item Agents can join only from the front.
		\item A switch happens when the leader leaves or a new agent joins.
	\end{itemize}
\end{algorithm}
This mechanism has no fairness guarantees for a single game, only over an infinite time horizon where each agent can have multiple, non-overlapping availability periods. We assume that each agent's different availability periods are independent and identically distributed (i.i.d). We also assume that for each availability period, the arrival rate of other agents is i.i.d.

\begin{thm}\label{prop:ergodity}
	Given the i.i.d assumption of the availability periods and arrival rates, and assuming that each agent participates in an infinite number of convoys, if all agents use the Repeated Game Load Balancing mechanism, every agent will lead for the expected ex-post proportional share.
	

\end{thm}

The proof can be found in Appendix C. 

The main advantage of this mechanism is that no switching is performed and thus no switching cost is incurred.

\begin{cor}\label{prop:equitability}
	Given the i.i.d assumption of the availability periods and arrival rates, the Repeated Game Load Balancing mechanism guarantees equitability, i.e. ex-post proportionality, in expectation. 
\end{cor}


While the Repeated Game Load Balancing mechanism is fair in expectation,
any individual participant may end up with a very unfair allocation until it has participated in many instances. To study this effect, and in particular how many convoys an agent needs to participate in, in order for its ratio of actual to ex-post proportional lead time to converge to 1, we created a custom simulation environment.
\footnote{The code is attached as supplementary material and will be made publicly available upon acceptance.}
The details of the experiments can be found in Appendix D, and the results indicate that only after participating in 700 convoys on average, there are no vehicles who lead more than $10\%$ of their ex-post proportional share. 

\subsubsection{Single Game Load Balancing}
\label{sec:single_game}


Creating a distributed single game mechanism that guarantees fairness, and maximizes efficiency is challenging, since for a distributed, single SOCD game, equitability (i.e., ex-post proportionality) cannot be guaranteed, as described in Theorem \ref{thm:imposibility}. As a result, we aim to guarantee ex-ante proportionality, while getting as close as possible to ex-post proportionality in practice. The mechanism we propose requires each agent to lead the convoy for no more than its ex-ante proportional share, and rotate at most once.

At any given time $t$, the convoy consists of agents that have finished leading their share, denoted as ${A_f}^t$, and agents who have yet to complete their assigned share, denoted as ${A_l}^t$. The agents in ${A_l}^t$ are sorted by their leave times with the first agent to leave located at the front of the convoy. 
According to the mechanism we propose, a new agent, $a_i$, entering the convoy at time $t\_arrive_i$ is inserted into the convoy sorted by its leave time, among the agents in ${A_l}^{t\_arrive_i}$ (for brevity we denote this set as ${A_l}^i$, and ${A_f}^{t\_arrive_i}$ as ${A_f}^i)$.

Leading agents lead until someone else joins in front of them, or until they finish their assigned leading share, after which they rotate to the back of the convoy. 
As a result, the agents in ${A_f}^t$ occupy the back of the convoy and do not have to lead again until they leave, thus guaranteeing that agents rotate at most once. 

When an agent $a_i$ enters, it is assigned the ex-ante proportional share with regards to all the agents in the convoy, (${A_l}^i$ and  ${A_f}^i$). 
Regarding the existing agents there are two possibilities to consider:
\begin{enumerate}[label=\roman*,leftmargin=5pt]

\item The existing agents' assigned leading shares will remain unchanged when a new agent enters.

\item  The remaining shares of the agents in ${A_l}^i$ will be dynamically adjusted according to the contribution of the new agent. Note that the shares of the agents in ${A_f}^i$ cannot be changed since performed tasks are irrevocable. 

The assigned share of agents in ${A_l}^i$ will be adjusted in the following way. For each segment $seg_j^i \in EAS^i$ in $a_i$'s availability period, $a_i$'s share will be $ \frac{ |seg_j^i|}{ n\_seg_j^i}$, and this added share will be divided among the agents from ${A_l}^i$ which are available in that segment, ${A_l}_j^i$. The agents in ${A_l}_j^i$ will each reduce their remaining shares $rm$ by $ {\frac{ |seg_j^i|}{ n\_seg_j^i}}/{|{A_l}_j^i|}$ with a minimum of 0.

\end{enumerate}
\vspace{-5pt}
\begin{algorithm}
\small
\caption{Single Game Load Balancing Mechanism}
\label{alg:sglb}
\begin{algorithmic}[1]
\State Agents in $A_l$ are sorted by order of leave time with the first one to leave in front.
\State New agents join at the sorted location in $A_l$ according to their leave time.
\State Each agent is initially allocated its ex-ante proportional share to lead.
\If{DynamicAdjustment}
	\For{(Every new entrant $a_i$ )}
        \For{(Every segment $seg_j^i \in EAS^i$ )}
            \For{(Every agent $a_k$ in  ${A_l}_j^i$ )}
             \State $rm_k = \max\{0, rm_k - {\frac{ |seg_j^i|}{ n\_seg_j^i}}/{|{A_l}_j^i|}\}$
      \EndFor
      \EndFor
\EndFor
\EndIf
\item A switch happens when the leader finishes its share or when a new agent joins in front of it.
\item When a leader finishes its share it rotates to the back of the convoy.
\end{algorithmic}
\end{algorithm}

While both variations of this mechanism (with or without dynamic adjustment) guarantee ex-ante proportionality, it is interesting to measure how close they come to the equitable ex-post proportionality in practice. For that reason we ran simulations of convoys on a straight 100km long highway with 100 uniformly distributed entry/exit stations. A single convoy moves through the stations and agents join it at their entry station and exit at their destination station. We consider two configurations. 
In configuration1 entry stations are randomly sampled from [1,100], and exit stations are sampled from (entry,100], uniformly at random.  In configuration2, $20\%$ of the agents sample the entry and exit stations as in configuration1, and the other $80\%$  use a bi-modal distribution of entry and exit stations where entry stations are sampled from the first 10 stations, and exit stations are sampled from the last 10 stations uniformly at random. The second configuration models a more realistic scenario where there is a highway between two city hubs. For each configuration we ran 100 convoys with 10 agents in each. For each participation of an agent in a convoy we measure the following quantities: \begin{itemize}[leftmargin=*]
    \item The ex-post proportional share (EPPS) that would result from employing the Payment Transfer mechanism.
    \item The actual leading share when employing the Repeated Game (RG) mechanism, and its ratio to EPPS.
    \item The actual leading share when employing the Single Game (SG) mechanism without dynamic adjustment and its ratio to EPPS.
    \item The actual leading share when employing the Single Game mechanism with Dynamic Adjustment (SG\_DA) and its ratio to EPPS.
\end{itemize}
For all of the ratios we computed the Gini coefficient which is a measure of statistical dispersion intended to represent the inequality within a group of people \cite{sen1997economic}. A Gini coefficient of zero expresses perfect equality, and a Gini coefficient of one expresses maximal inequality.

\begin{table}[]
\centering
\begin{tabular}{ccccl}
\cline{1-4}
\multicolumn{1}{|c|}{Configuration}                       & \multicolumn{1}{c|}{RG}   & \multicolumn{1}{c|}{SG}   & \multicolumn{1}{c|}{SG-DA}         &  \\ \cline{1-4} 
\multicolumn{1}{|c|}{1. Uniform station distribution}  & \multicolumn{1}{c|}{0.55} & \multicolumn{1}{c|}{0.44} & \multicolumn{1}{c|}{\textbf{0.06}}          &  \\ \cline{1-4}
\multicolumn{1}{|c|}{2. Bi-modal station distribution} & \multicolumn{1}{c|}{0.80} & \multicolumn{1}{c|}{0.69} & \multicolumn{1}{c|}{\textbf{0.02}} &  \\ \cline{1-4}
\multicolumn{1}{l}{}                         & \multicolumn{1}{l}{}      & \multicolumn{1}{l}{}      & \multicolumn{1}{l}{}               & 
\end{tabular}
\vspace{-20pt}
\caption{Gini coefficient comparison.}
\vspace{-15pt}
\end{table}

The results indicate that the SG mechanism provides allocations which are in practice more equal and closer to the EPPS than those provided by the RG mechanism (0.44 compared to 0.55). This result is expected since the RG mechanism is designed to converge to EPPS after multiple participations, and since it does not require rotations it is likely that some agents lead for much more than their EPPS in a single convoy. The inequality of RG becomes more acute in the bi-modal station distribution scenario with a Gini score of 0.80, since the leader has fewer chances of being replaced in the section between the two cities, which causes great inequality for a single game. However the SG mechanism's equality also deteriorates in the bi-modal station distribution, due to the fact that earlier entrants are allocated relatively large leading shares which are not adjusted as new agents arrive. 

The SG-DA mechanism has a much better Gini score than the SG mechanism in the uniform setting (0.06 compared to 0.44) and surprisingly it even improves its score in the bi-modal distribution where the other two mechanisms got worse. The reason that SG-DA improves in the bi-modal distribution to almost perfect equality with a Gini score of 0.02 is that since all the agents enter in the early stages of the convoy, the mechanism can adjust their shares before anyone has lead for a significant portion of their allocated share, when there are fewer agents in $A_f$ compared to the uniform distribution setting. 

\section{Conclusions and Future Work}
\label{sec:conc}
In this paper we define the SOCD problem, a novel sequential online variation of the chore division problem. We instantiate SOCD on a real-world problem of autonomous vehicle convoy formation.
We propose three fair-division mechanisms to balance the load of the leader and equally share the energy savings of the followers among all the convoy's participants.
The Payment Transfer mechanism assumes the existence of a central payment transfer system, and achieves optimal efficiency and fairness. 
The Repeated Game Load Balancing mechanism does not rely on a central payment system, yet offers optimal efficiency, and fairness in expectation, after participating in multiple repeated games. 
The Single Game Load Balancing mechanism is also distributed, and is able to achieve ex-ante proportionality with a minimal number of rotations. A variation of this mechanism which dynamically adjusts the allocated shares achieves in practice allocations that are very close to fair in terms of ex-post proportionality, and a Gini coefficient score of 0.02, in a realistic highway setting.

Possible threads for future work include designing mechanisms that can support heterogeneous valuation functions. There are multiple types of heterogeneity which can result in alternate definitions of fairness, and more complex requirements when optimizing for efficiency, as discussed in Appendix A.

\newpage

\bibliographystyle{named}
\bibliography{references}

\section{Acknowledgments}
This work has taken place in the Learning Agents Research
Group (LARG) at UT Austin.  LARG research is supported in part by NSF
(CPS-1739964, IIS-1724157, NRI-1925082), ONR (N00014-18-2243), FLI
(RFP2-000), ARO (W911NF-19-2-0333), DARPA, Lockheed Martin, GM, and
Bosch.  Peter Stone serves as the Executive Director of Sony AI
America and receives financial compensation for this work.  The terms
of this arrangement have been reviewed and approved by the University
of Texas at Austin in accordance with its policy on objectivity in
research.

\appendix

\section{Appendix for section 3 - Heterogeneous Agents}
\label{appendix3}

Considering heterogeneous agents is a natural direction for future work, but presents new challenges that would require more space than is available in this initial paper that introduces this new problem.
Just to give some insight, heterogeneity can manifest in different ways: 
\begin{enumerate}[label=\roman*]
	\item Each agent $a_i$ has an individual utility per unit of time $u_i$ as a follower. However, the utility per unit of time is constant. This case allows representing vehicles with different (static) energy saving profiles.  
	\item Each agent $a_i$ has an individual utility per unit of time $u_i(a_{front})$ as a follower. However, the utility per unit of time is a function of the vehicle in front $a_{front}$. This case allows representing vehicles with different energy saving profiles which are a function of the vehicle in front of them. For example, a small vehicle enjoys large savings when following a large truck, but if the order is reversed, the truck does not save much energy if it follows a small vehicle.
	\item Each agent $a_i$ has an individual cost per unit of time for leading, $c_i$. For example one soccer player might dislike being the goalie more than others, and would be willing to pay more for others to be active.
	
	\item Each agent $a_i$ has an individual valuation $v_i$ for leading different road segments. For example, a lookout in a campground might prefer to take the first watch to avoid waking up in the middle of the night. 
	\item We can also consider some combinations of these assumptions.
\end{enumerate}
When considering these options, the modeling of the problem changes and the notion of fairness becomes ill-defined. 

When considering case (i) for example, we need to reconsider the concept of proportionality and define what should be the proportional share in this case? Should it be $1/n$ of the follower's own utility, or $1/n$ of the leader’s utility? 

The fairness and optimality guarantees of the mechanisms depend on the definitions of fairness and will have to be reexamined in case it changes. 

Regarding Efficiency, the solution mechanisms will also need to be modified. For example, in the homogeneous case, the Payment Transfer mechanism allows any agent to be the leader, but in the heterogeneous case (i), the mechanism needs to be optimized by assigning the agent with the lowest utility to be the leader, in order to maximize the total followers’ savings.

In this paper we restrict our attention to considering homogeneous agents, and lay the foundations to the analysis of heterogeneous utility and cost functions in several possible avenues of future work.

\section{Appendix for Section 6.1 - Optimality of Payment Transfer Mechanism}
\label{appendix6.1}
Proof for Theorem 2. 

The Payment Transfer mechanism is efficiently optimal, and equitable.

\begin{proof}
	Since no rotations are made in the \textit{Payment Transfer} mechanism, it achieves optimal efficiency.
	In terms of fairness, for each segment, all the followers save the same amount of energy and pay the same amount of money to the leader. All we need to prove is that the sum of payments that the leader receives per segment is equal to each of the followers' saving for that segment, calculated as:
	$$
	|seg|\cdot u-p_{seg}^i =|seg|\cdot u-\frac{|seg|\cdot u}{n_{seg}}
	=\frac{(n_{seg}-1)\cdot|seg|\cdot u}{n_{seg}}
	$$	
	The leader's received payment for every segment is equal to:	
	$$
	(n_{seg}-1)\cdot p_{seg}^i =(n_{seg}-1)\cdot \frac{|seg|\cdot u}{n_{seg}}\\
	=\frac{(n_{seg}-1)\cdot|seg|\cdot u}{n_{seg}}
	$$
	which is identical to the saving of each follower and thus the mechanism guaranties equitability, both ex-ante and ex-post,
	and allows agents to be indifferent between leading and following. 
\end{proof}

\section{Appendix for Section 6.2.1 - Fairness of Repeated Game Load Balancing Mechanism}
\label{app:appendix6.2}
Proof for Theorem 3. \ref{prop:ergodity}:
\begin{thm}\label{prop:ergodity}
	Given the i.i.d assumption of the availability periods and arrival rates, and assuming that each agent participates in an infinite number of convoys, if all agents use the Repeated Game Load Balancing mechanism, every agent will lead for the expected ex-post proportional share.
	
\end{thm}

\begin{proof}
	
	If all the agents use the Repeated Game Load Balancing mechanism in repeated games, due to the law of large numbers, and the assumption that availability periods and arrival rates are i.i.d, the average of the results obtained from a large number of convoy formations is equal to the expected value. Each agent has the same probability for leading as all the others, hence, in expectation they will all lead for the same amount of time and follow for the same amount of time in any size of convoys.
\end{proof}

\section{Appendix for Section 6.2.2 - Experiments of Repeated Game Mechanism}
\label{app:appendix6.2.2}

While the Repeated Game Load Balancing mechanism is fair in expectation, for realistic applications such as convoy formation, it is important to get a sense of how fast does the actual lead time converge to the ex-post proportional lead time, i.e, how many convoys does an agent need to participate in, in order for its ratio of actual to ex-post proportional lead time to converge to 1. For that reason we created a software simulation environment. 

\begin{figure}[h]
	\centering
	\includegraphics[scale=0.33]{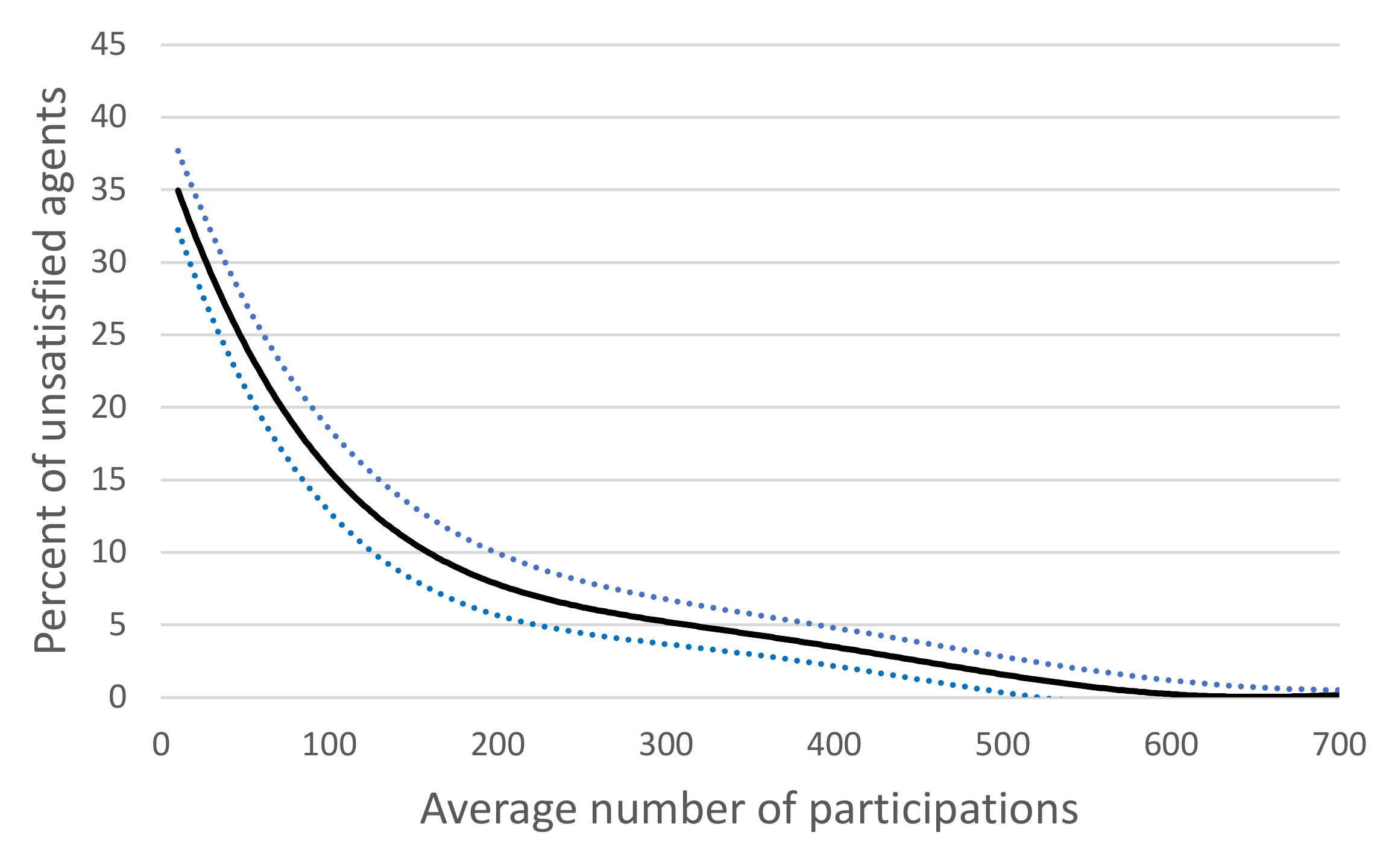}
	\caption{Simulation results for the Repeated Game Load Balancing mechanism.}
	\label{fig:ratio}
\end{figure}

The experimental setup has 100 stations where vehicles can enter and exit. The stations are uniformly distributed on a ring road with a total length of 100km. 
100 vehicles are randomly distributed over these 100 stations (using a uniform distribution).
A single convoy cycles through the stations, and when it reaches a station, the vehicles in that station have a 0.1 probability of entering the road and joining the convoy.
The distance they join for is also randomly generated from a uniform distribution $[0,100]$. Once a vehicle completes its distance, it leaves the convoy and parks at its destination station until the convoy passes it again, at which point it has a 0.1 probability of rejoining.
If multiple vehicles join at the same station, their order is randomized since the last one in becomes the leader.
Leaders are replaced when another vehicle joins or when they reach their destination. 
For every section between consecutive stations, every vehicle in the convoy accumulates $\frac{1}{n}$ to its proportional share, and the leader also accumulates 1 to its actual share.
Whenever a vehicle exits, the accumulated actual and ex-post proportional shares are recorded, and added to the list of convoys that the vehicle has participated in.
We measure how many participations on average it takes for the ratio of actual over ex-post proportional share to converge to 1. 
Specifically, we measure what is the percentage of vehicles that lead for $10\%$ or more than what they should have, which we refer to as \emph{unsatisfied agents}.

The results indicate that it takes at least 200 participations per vehicle on average to get fewer than $10\%$ unsatisfied agents.
After participating in 700 convoys on average, there are no vehicles who lead more than $10\%$ of their ex-post proportional share. 
The results are shown in Figure \ref{fig:ratio}. The dashed lines represent one standard deviation above and below the mean.

\end{document}